\documentclass[11pts]{article}

\DeclareSymbolFont{extraup}{U}{zavm}{m}{n}
\DeclareMathSymbol{\varheartsuit}{\mathalpha}{extraup}{86}
\DeclareMathSymbol{\vardiamondsuit}{\mathalpha}{extraup}{87}

\usepackage{fullpage}
\usepackage[utf8]{inputenc} 
\usepackage{hyperref}       
\usepackage{url}            
\usepackage{booktabs}       
\usepackage{amsfonts}       
\usepackage{nicefrac}       
\usepackage{microtype}      
\usepackage{lipsum}
\usepackage{algorithm}
\usepackage{datetime}
\usepackage{amsthm}
\usepackage{color}
\usepackage{mathrsfs}
\usepackage{amsmath}
\usepackage{amssymb}
\usepackage[sort,nocompress]{cite}

\usepackage[dvipsnames]{xcolor}
\usepackage{pifont}
\usepackage{mathtools}
\newcommand{\eqdef}{\coloneqq}

\newcommand{\norm}[1]{\left\| #1 \right\|}

\newcommand{\cC}{{\cal C}}
\newcommand{\cT}{{\cal T}}

\newcommand{\R}{\mathbb{R}}

\newcommand{\bR}{\mathbb{R}}

\newcommand{\bP}{\mathbb{P}}

\newcommand{\mcG}{\mathscr{G}}
\newcommand{\mcF}{\mathscr{F}}

\newcommand{\mS}{\mathcal{S}}

\DeclareMathOperator{\prox}{prox}

\DeclareMathOperator*{\argmin}{arg\,min}

\newcommand\ec[2][]{\ensuremath{\mathbb{E}_{#1} \left[#2\right]}}
\newcommand\ecn[2][]{\ec[#1]{\norm{#2}^2}}
\newcommand{\sqn}[1]{{\left\lVert#1\right\rVert}^2}
\newcommand\br[1]{\left ( #1 \right )}
\newcommand\pbr[1]{\left \{ #1 \right \} }
\newcommand\ev[1]{\left \langle #1 \right \rangle}

\usepackage{algpseudocode}
\usepackage{mathtools}

\usepackage[colorinlistoftodos,bordercolor=orange,backgroundcolor=orange!20,linecolor=orange,textsize=scriptsize]{todonotes}

\newtheorem{theorem}{Theorem}

\newtheorem{assumption}{Assumption}

\newtheorem{lemma}{Lemma}

\usepackage{hyperref}
\hypersetup{
    colorlinks=true,
    linkcolor=blue,
    filecolor=magenta,
    urlcolor=cyan,
}

\title{\bf Distributed Fixed Point Methods\\  with Compressed Iterates\footnote{This paper was prepared during Summer 2019 when the first two authors were research interns at KAUST. }}

 \author{S\'{e}lim Chraibi$^{\vardiamondsuit \clubsuit}$ \qquad Ahmed Khaled$^{\vardiamondsuit \spadesuit}$ \qquad Dmitry Kovalev$^\vardiamondsuit$ \\ Peter Richt\'{a}rik$^\vardiamondsuit$ \qquad Adil Salim$^\vardiamondsuit$ \qquad Martin Tak\'{a}\v{c}$^\varheartsuit$  \\
 \phantom{xx}
 \\
 $^\vardiamondsuit$ King Abdullah University of Science and Technology \\
 $^\clubsuit$ Universit\'{e} Grenoble Alpes\\
 $^\spadesuit$ Cairo University\\
 $^\varheartsuit$ Lehigh University
 }
 
 \date{December 20, 2019}
 
\begin{document}
\maketitle





\begin{abstract}
We propose basic and natural assumptions under which iterative optimization methods with compressed iterates can be analyzed. This problem is motivated by the practice of federated learning, where a large model stored in the cloud is compressed before it is sent to a mobile device, which then proceeds with training based on local data. We develop standard and variance reduced methods, and establish communication complexity bounds. Our algorithms are the first distributed methods with compressed iterates, and the first fixed point methods with compressed iterates.

\end{abstract}

\section{Introduction}

Communication efficiency and memory issues often arise in machine learning algorithms dealing with large models. This is specifically true in federated learning, where a network of $n$ computing agent is required to jointly solve a machine learning task~\cite{Konecny16, caldas2018expanding}. In this case, dealing with large models implies suffering from a high communication time, which is known to be the bottleneck in distributed applications~\cite{alistarh2017qsgd}.

To overcome this issue, standard techniques proceed by compressing the gradients of some SGD (Stochastic Gradient Descent) like training algorithm as in \cite{alistarh2017qsgd, WenTernGrad17, bernstein2018signsgd}. We consider the case where the iterates themselves need to be compressed. This case is relevant even in the $n=1$ case, where there is only one computing agent provided that the model is too large to keep in memory and needs to be compressed. Training with compressed iterates was only considered in one recent work~\cite{GDCI}, which introduced the gradient descent algorithm with compressed iterates. In this paper we improve the results of~\cite{GDCI} and generalize them in two ways. First, in the case $n=1$, we consider iterates compression in any algorithm that can be formulated as a (stochastic) fixed point iteration. This covers gradient descent and stochastic gradient descent, among others. Second, we consider the distributed case $n \geq 2$, where the network has to jointly find a fixed point of some map, in a distributed manner over the nodes, and using iterate compression. This distributed fixed point problem covers many applications of federated learning, including distributed minimization or distributed saddle point problems.

To address these problems we first study a naive approach that relies on compressing the iterates after each iteration. This iterates compression introduces an extra source of variance in the algorithms. We then propose a variance reduced approach that allows to remove the variance induced by the compression.




In summary, we make the following contributions:
\begin{itemize}
    \item We propose new distributed algorithms (non variance reduced and variance reduced) to learn with compressed iterates in the fixed point framework, which we show captures gradient descent as well as a variety of other methods.
    \item We derive non asymptotic convergence rates for these methods. Our theory allows improved rates when specialized to gradient descent compared to prior work, and captures other variants of gradient descent that have not been previously studied.
    \item We experiment numerically with the developed algorithms on synthetic and real datasets and report our findings.
\end{itemize}


\subsection{Related Work}
\textbf{Communication-Efficiency}. In distributed optimization, the communication cost is the bottleneck. In order to reduce it, many methods have been suggested including the use of intermittent communication and decentralization \cite{Wang18}, as well as exchanging only compressed or quantized information between the computing units~\cite{tsitsiklis1987communication}. Usually, the exchanged information is usually some compressed gradient~\cite{seide20141,alistarh2017qsgd,bernstein2018signsgd} or compressed model update~\cite{Basu2019, reisizadeh2019fedpaq} in a distributed master-worker setting. We note that in the setting of gradient compression, various methods have also been developed to reduce the noise from gradient compression and the method we develop is similar in spirit to some of them, as \cite{DIANA2, mishchenko2019distributed}. As noted before, iterate compression is also used in Federated Learning, see e.g.\ \cite{Konecny16,caldas2018expanding}, where concerns of communication efficiency and memory usage are particularly important.

\textbf{Decentralized Methods}. In decentralized settings, one can distinguish between exact methods and approximate methods. Among inexact methods, some provide algorithms that compute (sub)gradients at compressed iterates: \cite{nedic2008distributed,rabbat2005quantized}. Exact methods usually exchange compressed gradients or compressed iterates \cite{koloskova2019decentralized,doan2018accelerating,reisizadeh2018quantized,berahas2019nested,zhang2019compressed,lee2018finite}. Our focus in this work is on centralized methods, and we leave an extension to decentralized settings to future work.

In addition to concerns of communication efficiency, because compression operators satisfying Assumption~\ref{asm:compression-operator} are not all necessarily quantization operators, our method (in the $n=1$ case) can also be seen as an analysis of fixed point methods with perturbed iterates in a similar spirit to \cite{Devolder2014}, who analyze gradient descent methods given access to an inexact oracle.

The remainder is organized as follows. In the next section we provide some background on distributed fixed point problems and compression operators, and make our assumptions. In Section~\ref{sec:n=1}, we consider the case $n=1$ where there is only one computing unit. We describe our algorithms, state the main results and instantiate the algorithms to practical (stochastic) fixed point iterations. The case $n=1$ is generalized in Section~\ref{sec:n>1} where a network of computing units is considered. We describe our distributed algorithms, state the main results and instantiate the distributed algorithms to practical distributed (stochastic) fixed point iterations. Finally, simulations on a federated learning task is provided in Section~\ref{sec:num}. The proofs of our theorems are postponed to the appendix.

\section{Background}
\subsection{Distributed fixed point}

Let $\cT_1, \cT_2,\dots, \cT_n$ be operators on $\R^d$, i.e., $\cT_i : \bR^d \to \bR^d$. Denoting
\begin{equation}
    \label{eq:finite-sum}
    \cT(x) \eqdef \frac{1}{n} \sum_{i=1}^{n} \cT_i (x),
\end{equation}
our goal is to find a fixed point of $\cT$, i.e., a point $x^\star$ such that
\begin{equation}
\label{eq:fixed_point}
\cT(x^\star) = x^\star .
\end{equation}
Consider a probability space $(\Omega, \mcF, \bP)$, a family $s := (s_i)_{i \in \{1,\ldots,n\}}$ of random variables defined on $(\Omega, \mcF, \bP)$ with values in some measurable space $(\Xi,\mcG)$. Denote $\mS_i$ the distribution of $s_i$ and $\mS$ the distribution (over $\Xi^n$) of $s$. We allow each $\cT_i$ to have the following stochastic representation:
\begin{equation}
    \cT_i (x) = \ec[s_i]{\cT_i (x, s_i)},
\end{equation}
where, with a small abuse of notation, $\cT_i(x,\cdot)$ denotes an $\mS_i$-integrable function for every $x \in \bR^d$. We also denote for every $x \in \bR^d$,
\begin{equation}
    \cT(x, s) \eqdef \frac{1}{n} \sum_{i=1}^{n} \cT_i (x, s_i).
\end{equation}
Note that $\cT(x,\cdot)$ is $\mS$-integrable and that $\ec[s]{\cT(x, s)} = \cT(x)$.

We assume the following contraction property for the stochastic map $\cT(\cdot,s)$.

\begin{assumption}
    \label{asm:T-contraction}
    There exist $x^\star \in \bR^d$, $B \geq 0$ and $\rho \in (0, 1)$ such that for every $x \in \bR^d$,
    \begin{equation}
        \label{eq:T-shrinkage-to-constant}
        \ecn[s]{\cT(x, s) - x^\star} \leq \br{1 - \rho} \sqn{x - x^\star} + B.
    \end{equation}
\end{assumption}
 This assumption is satisfied by many maps $\cT(\cdot,s)$ describing (stochastic) optimization algorithms under some strong convexity / smoothness assumption; see Sections~\ref{sec:n=1} and~\ref{sec:n>1}. We shall also use the expected Lipschitz continuity of $\cT_i(x,s)$ defined as follows.
 \begin{assumption}
    \label{asm:individual-Ti-Lipschitz}
    For every $i \in \{1,\ldots,n\}$, there exists $c_i \geq 0$ such that for every $x,y \in \bR^d$:
    \begin{equation}
        \label{eq:individual-Ti-Lipschitz}
        \ecn[s]{\cT_i (x, s) - \cT_i (y, s)} \leq c_i \sqn{x - y},
    \end{equation}
    and we denote
    \begin{equation*}
        c^2 \eqdef \frac{1}{n} \sum_{i=1}^{n} c_i^2.
    \end{equation*}
\end{assumption}

\subsection{Compression operator}
In order to overcome communication issues, we apply a compression operator to the iterates.

Consider a family $\xi := (\xi_i)_{i \in \{1,\ldots,n\}}$ of random variables defined on $(\Omega, \mcF, \bP)$ with values $(\Xi,\mcG)$. If $n=1$, we shall prefer the notation $\xi$ for $\xi_1$. We consider a measurable map $\cC : \bR^d \times \Xi \to \bR^d$ such that for every $i \in \{1,\ldots,n\}$,
\begin{equation}
    x = \ec[\xi_i]{\cC(x,\xi_i)}.
\end{equation}
The map $\cC$ is called a compression operator. We make the following assumption on $\cC$.
\begin{assumption}
    \label{asm:compression-operator} There exists $\omega\geq 0$ such that
    for every $i \in \{1,\ldots,n\}$ and every $x \in \bR^d$,
    \begin{equation}
        \ecn[\xi_i]{\cC(x; \xi_i) - x} \leq \omega \sqn{x}.
    \end{equation}
\end{assumption}
Assumption~\ref{asm:compression-operator} has been used before, either in this general form or in special cases, in the analysis of gradient methods with {\em compressed gradients}~\cite{koloskova2019decentralized,DIANA2} and {\em compressed iterates}~\cite{GDCI}. Many practical compression operators satisfy this assumption; e.g.,  natural compression and natural dithering,  standard dithering, sparsification, and quantization  \cite{horvath2019, DIANA2, GDCI, Stich18}.


\section{Results in the case $n=1$}
\label{sec:n=1}
In this section, we present two algorithms to solve \eqref{eq:fixed_point} in the case when $n=1$ and state two theorems related to these algorithms.

Consider stochastic fixed point iterations of the form
\begin{equation}
\label{eq:sto-fix-pt}
x^{k+1} = \cT(x^k,s^k),
\end{equation}
where $s^k$ is a sequence of i.i.d.\  copies of $s$. Our first algorithm compresses all iterates $x^k$ for $k\geq 1$.
\begin{algorithm}[H]
    \caption{FPMCI: Fixed Point Method with Compressed Iterates}
    \begin{algorithmic}
        \State {\bf Initialization}: $x^0 \in \R^d$, $(\xi^k)$ i.i.d.\  copies of $\xi$, $(s^k)$ i.i.d.\  copies of $s$
        \For{$k=0,1,2,\ldots$}
        \begin{equation*}
            x^{k+1} = \cC(\cT(x^k,s^k),\xi^k).
        \end{equation*}
        \EndFor
    \end{algorithmic}
    \label{eq:algo-singlenode-nonvr}
\end{algorithm}
Theorem~\ref{thm:singlenode-nonvr} states the convergence result obtained for Algorithm~\ref{eq:algo-singlenode-nonvr}.
\begin{theorem}
    \label{thm:singlenode-nonvr}
    Suppose that Assumptions~\ref{asm:T-contraction}, \ref{asm:individual-Ti-Lipschitz} and \ref{asm:compression-operator} hold. Let $r^k\eqdef \norm{x^{k} - x^\star}^2$. Then the iterates defined by Algorithm~\ref{eq:algo-singlenode-nonvr} satisfy
\[
         \ec{r^k}         \leq \br{ 1 - \rho + 2 \omega c^2 }^k r^0 + \frac{B + 2 \omega \sigma^2}{\rho - 2 \omega c^2},
\]
    where $\sigma^2 \eqdef \ecn[s]{\cT (x^\star, s)}.$
\end{theorem}
The convergence rate of $(x^k)$ is linear up to a ball of squared radius
\begin{equation}
\label{eq:radius}
    \frac{B}{\rho - 2\omega c^2} + \frac{2\omega\sigma^2}{\rho - 2\omega c^2}.
\end{equation}
The first term is coming from Assumption~\ref{asm:T-contraction}. The value of $B$ is usually zero for deterministic fixed point maps $\cT$, see the next subsection. In this case, the first term of~\eqref{eq:radius} is zero. The presence of the second term is mainly a consequence of the variance of the compression operator. If $\omega = 0$ (no compression), then the second term is equal to zero.\footnote{Having $\sigma^2 = 0$ is hopeless, except in very particular cases like $\cT$ deterministic and $x^\star = 0$.}

In order to remove this variance term, we develop a variance reduced version of Algorithm~\ref{eq:algo-singlenode-nonvr} to solve~\eqref{eq:fixed_point}.

\begin{algorithm}[H]
    \caption{VR-FPMCI: Variance Reduced Fixed Point Method with Compressed Iterates}
    \begin{algorithmic}
        \State {\bf Initialization}: $x^0, h^0 \in \R^d$, $(\xi^k)$ i.i.d.\  copies of $\xi$, $(s^k)$ i.i.d.\  copies of $s$
        \For{$k=0,1,2,\ldots$}
        \begin{align*}
            \delta^{k+1} &= \cC(\cT(x^k,s^k) - h^k,\xi^k)\\
            h^{k+1} &= h^k + \alpha \delta^{k+1}\\
            x^{k+1} &= \br{1 - \eta} x^k + \eta \br{h^k + \delta^{k+1}}.
        \end{align*}
        \EndFor
    \end{algorithmic}
    \label{eq:algo-singlenode-vr}
\end{algorithm}

The improved convergence rate of Algorithm~\ref{eq:algo-singlenode-vr} is stated by the next theorem.
\begin{theorem}
    \label{thm:singlenode-vr}
    Let $\Psi^k$ be the following Lyapunov function:
    \[ \Psi^k \eqdef \sqn{x^k - x^\star} + \frac{4 \eta^2 \omega}{\alpha} \ecn[s]{h^k - \cT(x^\star, s^k)}. \]
    Suppose that Assumptions~\ref{asm:T-contraction}, \ref{asm:individual-Ti-Lipschitz} and \ref{asm:compression-operator} hold.
    Then the iterates defined by Algorithm~\ref{eq:algo-singlenode-vr} satisfy
    \begin{equation}
        \label{eq:thm-VR-main-convergence-single}
        \ec{\Psi^{k}} \leq \br{1 - \frac{\min \pbr{\alpha, \eta \rho}}{2} }^k \ec{\Psi^0} + \frac{2\eta B}{\min \pbr{\alpha, \eta \rho}},
    \end{equation}
    if the stepsizes $\alpha,\eta$ satisfy $$\alpha \leq \frac{1}{\omega + 1} \quad\text{and}\quad \eta = \min \pbr{ 1, \frac{\rho}{12 \omega c^2}}.$$
\end{theorem}
Therefore, Algorithm~\ref{eq:algo-singlenode-vr} converges linearly if $B=0$ and allows for arbitrarily large compression variance.

\subsection{Examples}
We now give some instances of our algorithms~\ref{eq:algo-singlenode-nonvr} and~\ref{eq:algo-singlenode-vr} by particularizing the map $\cT$.

\subsubsection*{Gradient Descent}
Consider an $L$-smooth $\mu$-strongly convex objective function $F: \R^d \to \R$ and a step-size $\gamma \in \left( 0, \frac{1}{L} \right]$. Then
\begin{equation}
    \label{eq:TGD}
    \cT_{\text{GD}} : x \mapsto x - \gamma \nabla F(x)
\end{equation}
satisfies Assumption~\ref{asm:T-contraction} with $\rho = \gamma \mu$ and $B = 0$, and Assumption~\ref{asm:individual-Ti-Lipschitz} with $c = 1$~\cite{bau-com-livre11}. As a result, for any compression operator $\cC$ satisfying Assumption~\ref{asm:compression-operator}, Theorem~\ref{thm:singlenode-nonvr} states that
\[
        \ec{r^k} \leq  \br{1 - \gamma \mu + 2 \omega}^{k} r^0 + \frac{2 \omega}{\gamma \mu - 2 \omega} \sqn{x^\star}.
\]
This result improves upon the result obtained in~\cite{GDCI} by requiring $\omega < \frac{1}{2 \kappa}$ rather than $\omega < \frac{1}{76 \kappa}$ while still guaranteeing convergence. Moreover, using Theorem~\ref{thm:singlenode-vr}, $\ec{r^k}$ converges linearly to zero, rather to a neighbourhood of the solution, if Algorithm~\ref{eq:algo-singlenode-nonvr} is applied.

\subsubsection*{Stochastic Gradient Descent (SGD)}
Consider a $\mu$-strongly convex objective function $F: \R^d \to \R$ and $g(\cdot,s)$ an unbiased estimate of $\nabla F$ ($\forall x\in \R^d, \ec[s]{g(x,s)} = \nabla F(x)$). Assume that there exists $L >0$ such that
$$\ec[s]{\norm{g(x, s) - g(y, s)}} \leq L \norm{x - y}.$$
Then, a simple calculation shows that Assumption~\ref{asm:individual-Ti-Lipschitz} is satisfied by the map
\[ \cT_{\text{SGD}} : (x,s) \mapsto x - \gamma g(x,s).\] It is also known that Assumption~\ref{asm:T-contraction} is satisfied, with $B>0$ in general, see e.g.~\cite{Gower2019}.

\subsubsection*{Proximal SGD}
One can generalize the previous example to the map
\[ \cT_{\text{prox-SGD}} : (x,s) \mapsto \prox_{\gamma H}(x - \gamma g(x,s)),\]
where $H$ is a convex, lower semicontinuous and proper function $\bR^d \to (-\infty,+\infty]$, and $\prox_{\gamma H}$ is the proximity operator of $\gamma H$ defined as
$$
\prox_{\gamma H}(x) \eqdef \argmin_{y \in \bR^d}  \left\{\frac12\|x-y\|^2 + \gamma H(y)\right\}.
$$
The map $\cT_{\text{prox-SGD}}$ also satisfies the Assumptions~\cite{atc-for-mou-14,atc-for-mou-17}.
A fixed point of $\cT_{\text{prox-SGD}}$ is a minimizer of $F+H$.
\subsubsection*{Davis-Yin splitting}
Davis-Yin splitting~\cite{davis2017three} is an optimization algorithm to minimize a sum of three convex functions $F+G+H$. It is a generalization of Gradient Descent, Proximal Gradient Descent, and Douglas Rachford algorithms~\cite{boyd2011distributed,bau-com-livre11} and it takes the form of fixed point iterations $x^{k+1} = \cT_{\text{DY}}(x^k)$. The map $\cT_{\text{DY}}$ satisfies Assumptions~\ref{asm:T-contraction} and~\ref{asm:individual-Ti-Lipschitz} with $B=0$ if at least one of $F,G$ or $H$ is  strongly convex and at least one of $G$ or $H$ is smooth~\cite{davis2017three}. Therefore Algorithm~\ref{eq:algo-singlenode-vr} converges linearly in this case.

\subsubsection*{Vu-Condat splitting}
Vu condat splitting~\cite{con-jota13,vu2013splitting} is an optimization algorithm to minimize a sum of three convex functions $F(x)+G(x)+H(Ax)$ where $A$ is a matrix. It is a generalization of Gradient Descent, Proximal Gradient Descent, and Douglas Rachford, ADMM and Chambolle-Pock algorithms~\cite{chambolle2011first} and it takes the form of fixed point iterations $x^{k+1} = \cT_{\text{VC}}(x^k)$. The map $\cT_{\text{VC}}$ satisfies Assumptions~\ref{asm:T-contraction} and~\ref{asm:individual-Ti-Lipschitz} with $B=0$ if $G$ is strongly convex and $H$ is smooth.
Therefore Algorithm~\ref{eq:algo-singlenode-vr} converges linearly in this case.

\subsubsection*{(Stochastic) Gradient Descent Ascent}

Consider a $\mu$-strongly convex-concave function $F: \R^d \times \R^d \to \R$ defined by $F : (x,y) \mapsto F(x,y)$, (strongly convex in $x$ and strongly concave in $y$) with $L$-Lipschitz continuous gradient.
Then, the map
\begin{equation}
    \label{eq:TGDA}
    \begin{split}
        \cT_{\text{GDA}}: &(x,y) \mapsto \\
        &(x,y)^T - \gamma (\nabla_x F(x,y),-\nabla_y F(x,y))^T,
    \end{split}
\end{equation}
satisfies Assumption~\ref{asm:individual-Ti-Lipschitz} and Assumption~\ref{asm:T-contraction} with $B=0$ if $\gamma$ is small enough. In this case, Algorithm~\ref{alg:VR-distr-fixed-point-compressed-iterates} will converge linearly to a saddle point $x^\star$ of $F$. This example can be generalized to the case where the gradient $(\nabla_x F(x,y),-\nabla_y F(x,y))^T$ is replaced by an unbiased estimate with the expected Lipschitz continuity property, in which case Assumption~\ref{asm:T-contraction} holds with $B \geq 0$ in general.

\section{The case $n > 1$}
\label{sec:n>1}
We now consider the case where $n$ computing agents are required to compute a fixed point of $\cT$, under the restriction that each node $i$ only have access to the "local" random map $\cT_i(\cdot,\xi_i)$. We solve this problem in a distributed master/slave setting, where each iteration is divided into a computation step and a communication step. During the computation step, every node $i$ uses $\cT_i(\cdot,\xi_i)$ to update some "local" variable. Then, during the communication step, each node sends its local variable to the master node of the network that aggregates the variables and sends back the result to the other nodes.
We extend Algorithm~\ref{eq:algo-singlenode-nonvr} (resp.\ Algorithm~\ref{eq:algo-singlenode-vr}) to this setting, as well as Theorem~\ref{thm:singlenode-nonvr} (resp.\ Theorem~\ref{thm:singlenode-vr}). The distributed (non variance reduced) fixed point algorithm is summarized in Table~\ref{alg:distr-fixed-point-compressed-iterates}.

\begin{algorithm}[H]
    \caption{Distributed Fixed Point Method with Compressed Iterates}
    \begin{algorithmic}
        \State {\bf Initialization}: $x^0 \in \R^d$, $(\xi^k)$ i.i.d.\  copies of $\xi$, $(s^k)$ i.i.d.\  copies of $s$
        \For{$k=0,1,2,\ldots$}
        \State Broadcast $x^k$ to all nodes.
        \For{$i = 1, \ldots, n$ in parallel}
        \State Communicate to master node
        \[ \delta_i^{k+1} = \cC \br{\cT_i (x^k, s_i^k); \xi_i^k} \]
        \EndFor
        \State Compute
        \begin{equation}
            \label{eq:dfmci-update}
            x^{k+1} = \frac{1}{n} \sum_{i=1}^{n} \delta_i^{k+1}.
        \end{equation}
        \EndFor
    \end{algorithmic}
    \label{alg:distr-fixed-point-compressed-iterates}
\end{algorithm}
The convergence rate of this method is a direct generalization of Theorem~\ref{thm:singlenode-nonvr}
\begin{theorem}
    \label{thm:distributed-fpci-convergence}
    Let Assumptions~\ref{asm:T-contraction},~\ref{asm:individual-Ti-Lipschitz} and~\ref{asm:compression-operator} hold.     
    Assume moreover that $s_1,\ldots,s_n$ are independent. Let $r^k\eqdef \norm{x^{k} - x^\star}^2$.
    Then the iterates defined by Algorithm~\ref{alg:distr-fixed-point-compressed-iterates} satisfy
\[ \ec{r^k}
        \leq \br{ 1 - \rho + \frac{2 \omega c^2}{n} }^k r^0
        + \frac{B + \frac{2 \omega}{n} \sigma^2}{\rho - \frac{2 \omega c^2}{n}},
\]
    where $\sigma^2 \eqdef \frac{1}{n} \sum_{i=1}^{n} \ecn[s_i]{\cT_i (x^\star, s_i)}.$
\end{theorem}

Once again, the rate suffers from the variance term $\frac{\frac{2 \omega}{n} \sigma^2}{\rho - \frac{2 \omega c^2}{n}}$ which is removed by our variance reduced approach summarized in Table~\ref{alg:VR-distr-fixed-point-compressed-iterates}.


\begin{algorithm}[H]
    \caption{Distributed Variance-Reduced Fixed Point Method with Compressed Iterates}
    \begin{algorithmic}
        \State {\bf Initialization}: $x^0, h_1^0, h_2^0, \ldots, h_n^0 \in \R^d$, stepsize $\eta \in (0, 1]$, stepsize $\alpha > 0$
        \For{$k=0,1,2,\ldots$}
        \State Broadcast $x^k$ to all nodes.
        \For{$i = 1, \ldots, n$ in parallel}
         \begin{align*}
                \delta_i^{k+1} &= \cC (\cT_i (x^k, s_i^k) - h^k_i; \xi_i^k)\\
                h^{k+1}_{i} &= h^k_i + \alpha \delta^{k+1}_i\\
                \Delta^{k+1}_i &= \delta^{k+1}_i + h^k_i
                \end{align*}
        \State Communicate $\Delta^{k+1}_i$ to master node
            \EndFor
            \State Compute
            \begin{equation*}
                x^{k+1} = \br{1 - \eta} x^k + \eta \frac{1}{n} \sum_{i=1}^{n} \Delta^{k+1}_i.
            \end{equation*}
        \EndFor
    \end{algorithmic}
    \label{alg:VR-distr-fixed-point-compressed-iterates}
\end{algorithm}

Finally, the next theorem is the analogue of Theorem~\ref{eq:algo-singlenode-vr} in the distributed setting.
\begin{theorem}
    \label{theorem:VR-distributed-fpci}
    Define the Lyapunov function
    \begin{align*}
        \Psi^k \eqdef \sqn{x^k - x^\star} + \frac{4 \eta^2 \omega}{\alpha n^2} \sum_{i=1}^{n} \ecn[s_i]{h^k_i - \cT_i (x^\star, s^k_i)}.
    \end{align*}
    Suppose that Assumptions~\ref{asm:T-contraction}, \ref{asm:individual-Ti-Lipschitz} and \ref{asm:compression-operator} hold. Assume moreover that $s_1,\ldots,s_n$ are independent. Then the iterates defined by Algorithm~\ref{alg:VR-distr-fixed-point-compressed-iterates} satisfy
    \begin{equation}
        \label{eq:thm-VR-main-convergence}
        \ec{\Psi^{k}} \leq \br{1 - \frac{\min \pbr{\alpha, \eta \rho}}{2} }^k \ec{\Psi^0} + \frac{2\eta B}{\min \pbr{\alpha, \eta \rho}},
    \end{equation}
    if the stepsizes $\alpha,\eta$ satisfy $$\alpha \leq \frac{1}{\omega + 1} \quad\text{and}\quad \eta = \min \pbr{ \frac{\rho n}{12 \omega c^2}, 1 }.$$
\end{theorem}

Algorithm~\ref{alg:VR-distr-fixed-point-compressed-iterates} converges linearly if $B=0$. We further note that in the special case $\alpha \simeq 1$ the algorithm reduces to quantizing the \emph{model update} in expectation, a practice that is already common in practice. This further mirrors the result of \cite{DIANA2} where quantizing gradient differences rather than gradients allows several benefits over quantizing gradients. The message of Theorem~\ref{theorem:VR-distributed-fpci}, therefore, is that quantizing iterate differences rather than iterates also leads to better convergence properties.

\subsection{Examples}


\subsubsection*{Distributed (Stochastic) Gradient Descent}

Consider a $\mu$-strongly convex objective function $F: \R^d \to \R$ expressed as an empirical mean \[F(x) = \frac{1}{n}\sum_{i=1}^n f_i(x),\] where each $f_i$ is $L_i$-smooth and convex. Then it is easy to check that the map $\cT_{GD}$ defined in~\eqref{eq:TGD} takes the form~\eqref{eq:finite-sum} and that Assumptions~\ref{asm:T-contraction} and~\ref{asm:individual-Ti-Lipschitz} are satisfied by this map if $\gamma$ is small enough, see e.g.~\cite{Gower2019}. Algorithm~\ref{alg:VR-distr-fixed-point-compressed-iterates} is then a distributed gradient descent algorithm with iterates compression that converges linearly. If the $f_i$ are themselves written as expectations and have the expected Lipschitz continuity property and convexity, one can check that Assumptions~\ref{asm:T-contraction} and~\ref{asm:individual-Ti-Lipschitz} are also satisfied.

\subsubsection*{Distributed (Stochastic) Gradient Descent Ascent}

The example~\eqref{eq:TGDA} can be extended to the case where $F$ is expressed as an empirical mean if each term has a Lipschitz continuous gradient. In this case, the distributed Algorithm~\ref{alg:VR-distr-fixed-point-compressed-iterates} still converges linearly to a saddle point $x^\star$ of $F$.

\section{Empirical results}
\label{sec:num}

Here we present very preliminary numerical results. 

\paragraph{Experiment.} We minimize an $l_2$ regularized loss of a linear regression problem using gradient descent and natural compression \cite{horvath2019}. We carry out this experiment for different condition numbers.

\paragraph{Parametrisation.} We use an artificial regression dataset which allows us to have control over the conditioning of the loss function (by controling the singular values of the feature matrix). We use the stepsize $\gamma = \frac1L$ and in the variance reduced iteration we chose $\alpha = \frac{1}{1+\omega}$ and $\eta = \frac{\rho}{12\omega}$ (in the case of natural compression $\omega = \frac18$).

\paragraph{Results.} In the following we plot the evolution $ \| x^k - x^\star\|^2 $ for GD (Gradient Descent), GDCI (Gradient Descent with Compressed Iterates), VR-GDCI (Variance Reduced Gradient Descent with Compressed Iterates).

\begin{figure}[H]
    \centering
    \includegraphics[width=0.6\linewidth]{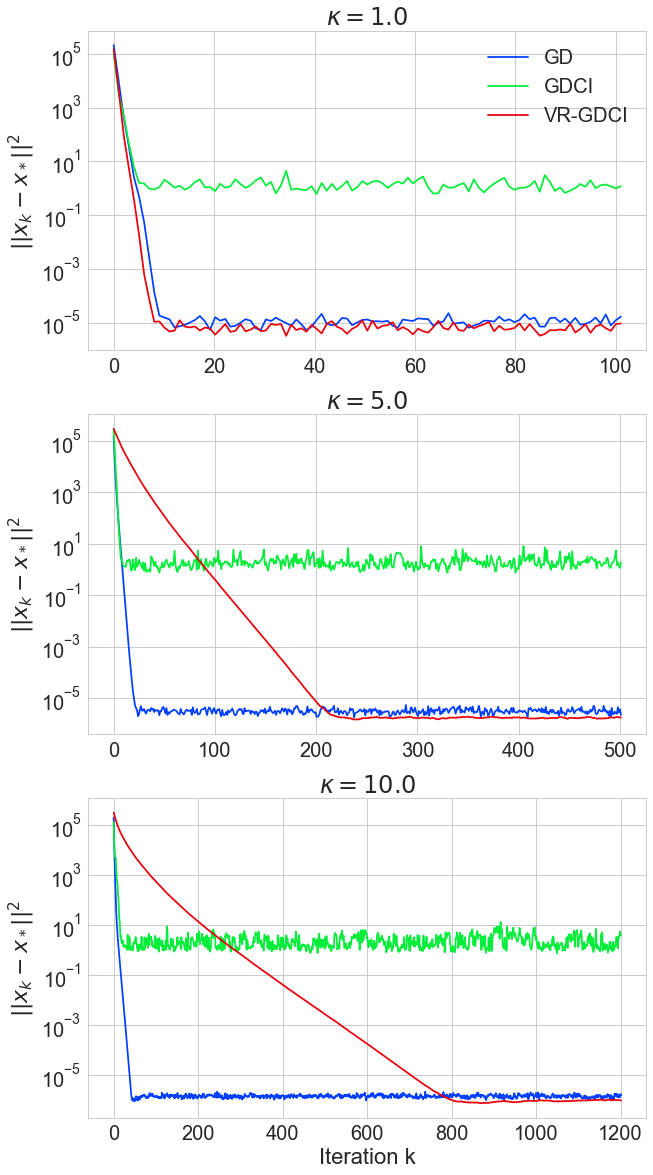}
    \label{fig:experiments}
\end{figure}


\bibliographystyle{plain}
\bibliography{bib}

\newpage
\appendix
\part*{Appendix}


\section{Basic Facts}

We recall the following fact about the variance of a random variable: Given a fixed $Y \in \R^d$ and a random variable $X \in \R^d$, we have
\begin{equation}
    \label{eq:basic-fact-variance-decomp}
    \ecn{X - Y} = \ecn{X - \ec{X}} + \sqn{\ec{X} - Y}.
\end{equation}
If $X_1, X_2, \ldots, X_n$ are independent random variables then
\begin{equation}
    \label{eq:variance-of-independent-sum}
    \ecn{\sum_{i=1}^{n} X_i - \ec{X_i}} = \sum_{i=1}^{n} \ecn{X_i - \ec{X_i}}.
\end{equation}
We also recall the following inequality from linear algebra: for any $a, b \in \R^d$ we have,
\begin{equation}
    \label{eq:sqnorm-triangle-inequality}
    \sqn{a + b} \leq 2 \sqn{a} + 2 \sqn{b}.
\end{equation}
We will also use the following fact: which follows from the convexity of the squared Euclidean norm: for $\eta \in [0, 1]$ we have,
\begin{equation}
    \label{eq:jensen-sqnorm}
    \sqn{\eta a + \br{1 - \eta} b} \leq \eta \sqn{a} + \br{1 - \eta} \sqn{b}.
\end{equation}

Moreover, we shall use the following lemma without mention.

\begin{lemma} \label{lem:recursion} Let $0<A<1$ and $B>0$  and let $\{r_k\}_{k\geq 0}$ be a sequence of real numbers with $r_0>0$ satisfying the recursion
\[ r_{k+1} \leq A r_{k} + B.\]
  Then
\[ r_k \leq A^k r_0 + \frac{B}{1-A}.\]
\end{lemma}



\section{Proof of Theorem~\ref{thm:distributed-fpci-convergence}}
Since Theorem~\ref{thm:singlenode-nonvr} is a particular case of Theorem~\ref{thm:distributed-fpci-convergence}, we only prove Theorem~\ref{thm:distributed-fpci-convergence}.
From \eqref{eq:dfmci-update}, we have conditionally on $(x^k,s^k)$,
\begin{eqnarray}
    \ecn{x^{k+1} - x^\star} &=& \ecn{\frac{1}{n} \sum_{i=1}^{n} \cC \br{\cT_i (x^k, s_i^k); \xi_i^k} - x^\star} \nonumber \\
    &\overset{\eqref{eq:basic-fact-variance-decomp}}{=}& \ecn{\frac{1}{n} \sum_{i=1}^{n} \cC \br{\cT_i (x^k, s_i^k); \xi_i^k} - \frac{1}{n} \sum_{i=1}^{n} \cT_i (x^k, s_i^k)} + \sqn{ \frac{1}{n} \sum_{i=1}^{n} \cT_i (x^k, s_i^k) - x^\star } \nonumber \\
    &\overset{\eqref{eq:finite-sum}}{=}& \frac{1}{n^2} \ecn{\sum_{i=1}^{n} \br{\cC \br{\cT_i (x^k, s_i^k); \xi_i^k} - \cT_i (x^k, s_i^k)}} + \sqn{\cT(x^k, s^k) - x^\star} \nonumber \\
    \label{eq:dfpci-proof-1}
    &\overset{\eqref{eq:variance-of-independent-sum}}{=} &\frac{1}{n^2} \sum_{i=1}^{n} \ecn{\cC\br{\cT_i (x^k, s_i^k); \xi_i^k} - \cT_i (x^k, s_i^k)} + \sqn{\cT(x^k, s^k) - x^\star}.
\end{eqnarray}
The first term in \eqref{eq:dfpci-proof-1} can be bounded using Assumption~\ref{asm:compression-operator}:
\begin{eqnarray}
    \frac{1}{n^2} \sum_{i=1}^{n} \ecn{\cC\br{\cT_i (x^k, s_i^k); \xi_i^k} - \cT_i (x^k, s_i^k)} &\leq & \frac{\omega}{n^2} \sum_{i=1}^{n} \sqn{\cT_i (x^k, s_i^k)} \nonumber \\
    &\overset{\eqref{eq:sqnorm-triangle-inequality}}{\leq} & \frac{2 \omega}{n^2} \sum_{i=1}^{n} \sqn{\cT_i (x^k, s_i^k) - \cT_i (x^\star, s_i^k)} + \frac{2 \omega}{n^2} \sum_{i=1}^{n} \sqn{\cT_i (x^\star, s_i^k)} \nonumber \\
    &\overset{\eqref{eq:individual-Ti-Lipschitz}}{\leq} & \frac{2 \omega}{n^2} \sum_{i=1}^{n} c_i^2 \sqn{x^k - x^\star} + \frac{2 \omega}{n^2} \sum_{i=1}^{n} \sqn{\cT_i (x^\star, s_i^k)} \nonumber \\
    \label{eq:dfpci-proof-2}
    &= & \frac{2 \omega c^2}{n} \sqn{x^k - x^\star} + \frac{2 \omega}{n^2} \sum_{i=1}^{n} \sqn{\cT_i (x^\star, s_i^k)}.
\end{eqnarray}
Plugging in \eqref{eq:dfpci-proof-2} in \eqref{eq:dfpci-proof-1},
\begin{eqnarray*}
    \ecn{x^{k+1} - x^\star} &\leq & \frac{2 \omega c^2}{n} \sqn{x^k - x^\star} + \frac{2 \omega}{n^2} \sum_{i=1}^{n} \sqn{\cT_i (x^\star, s_i^k)} + \sqn{\cT(x^k, s^k) - x^\star}
\end{eqnarray*}
Therefore, conditionally on $x^k$,
\begin{eqnarray*}
    \ecn{x^{k+1} - x^\star} &\leq & \frac{2 \omega c^2}{n} \sqn{x^k - x^\star} + \frac{2 \omega}{n^2} \sum_{i=1}^{n} \ecn{\cT_i (x^\star, s_i^k)} + \ecn{\cT(x^k, s^k) - x^\star} \\
    &\overset{\eqref{eq:T-shrinkage-to-constant}}{\leq} & \br{\frac{2 \omega c^2}{n} + 1 - \rho} \sqn{x^k - x^\star} + B + \frac{2 \omega}{n^2} \sum_{i=1}^{n} \ecn{\cT_i (x^\star, s_i^k)}.
\end{eqnarray*}
Finally taking unconditional expectations yields the theorem's claim.

\section{Proof of Theorem~\ref{theorem:VR-distributed-fpci}}
Since Theorem~\ref{thm:singlenode-vr} is a particular case of Theorem~\ref{theorem:VR-distributed-fpci}, we only prove Theorem~\ref{theorem:VR-distributed-fpci}.

\begin{lemma}
    \label{lemma:VR-hk-recursion}
    Under Assumption~\ref{asm:compression-operator}, if $0 < \alpha \leq \frac{1}{\omega + 1}$, then for every $i = \{1, \ldots, n\}$ the iterates of Algorithm~\ref{alg:VR-distr-fixed-point-compressed-iterates} satisfy conditionally on $x^k$ and $h^k_i$:
    \begin{align}
        \label{eq:lma-VR-hk-recursion}
        \ecn{h^{k+1}_i - \cT_i (x^\star, s_i^k)} \leq \br{1 - \alpha} \ecn{h^k_i - \cT_i (x^\star, s_i^k)} + \alpha \ecn{\cT_i (x^k, s_i^k) - \cT_i (x^\star, s_i^k)}
    \end{align}
\end{lemma}

\begin{proof}

Conditionally on $x^k, h^k_1, \ldots, h^k_n, s_1^k, \ldots, s_n^k$ we have
\begin{eqnarray*}
    \ecn{h^{k+1}_i - \cT_i (x^\star, s_i^k)} &=& \ecn{h^{k}_i - \cT_i (x^\star, s_i^k) + \alpha \delta^{k}_i} \\
    &=& \sqn{h^k_i - \cT_i (x^\star, s_i^k)} + 2 \alpha \ev{ h^k_i - \cT_i (x^\star, s_i^k), \ec{\delta^{k}_i} } + \alpha^2 \ecn{\delta^k_i} \\
    & \leq & \sqn{h^k_i - \cT_i (x^\star, s_i^k)} + 2 \alpha \ev{h^k_i - \cT_i (x^\star, s_i^k), \cT_i (x^k, s_i^k) - h^k_i} \\
    && \qquad + \alpha^2 \br{\omega + 1} \sqn{\cT_i (x^k, s_i^k) - h^k_i} \\
    & \leq & \sqn{h^k_i - \cT_i (x^\star, s_i^k)} + 2 \alpha \ev{h^k_i - \cT_i (x^\star, s_i^k), \cT_i (x^k, s_i^k) - h^k_i} \\
    && \qquad + \alpha \sqn{\cT_i (x^k, s_i^k) - h^k_i} \\
    &=& \sqn{h^k_i - \cT_i (x^\star, s_i^k)} + \alpha \ev{2 h^k_i - 2 \cT_i (x^\star, s_i^k) + \cT_i (x^k, s_i^k) - h^k_i, \cT_i (x^k, s_i^k) - h^k_i}.
\end{eqnarray*}
For the inner product in the last inequality, we have
\begin{align*}
    &\ev{2 h^k_i - 2 \cT_i (x^\star, s_i^k) + \cT_i (x^k, s_i^k) - h^k_i, \cT_i (x^k, s_i^k) - h^k_i} \\
    &= \ev{h^k_i - \cT_i (x^\star, s_i^k) + \cT_i (x^k, s_i^k) - \cT_i (x^\star, s_i^k), \cT_i (x^k, s_i^k) - \cT_i (x^\star, s_i^k) - \br{h^k_i - \cT_i (x^\star, s_i^k)}} \\
    &= - \sqn{h^k_i - \cT_i (x^\star, s_i^k)} + \sqn{\cT_i (x^k, s_i^k) - \cT_i (x^\star, s_i^k)}.
\end{align*}
Using this in the previous inequality, we get
\begin{align*}
    \ecn{h^{k+1}_i - \cT_i (x^\star, s_i^k)} &= \br{1 - \alpha} \sqn{h^k_i - \cT_i (x^\star, s_i^k)} + \alpha \sqn{\cT_i (x^k, s_i^k) - \cT_i (x^\star, s_i^k)}.
\end{align*}
It remains to take expectation with respect to the randomness in $s_i^k$.

\end{proof}

\begin{lemma}
    \label{lemma:VR-distance-recursion}
    Under Assumptions~\ref{asm:T-contraction} and \ref{asm:compression-operator}, the iterates of Algorithm~\ref{alg:VR-distr-fixed-point-compressed-iterates} satisfy,
    \begin{eqnarray}
            \ecn{x^{k+1} - x^\star} & \leq & (1 - \eta \rho) \sqn{x^k - x^\star} + \eta B \nonumber \\
            && \qquad + \frac{2 \eta^2 \omega}{n^2} \sum_{i=1}^{n} \ec{ \sqn{\cT_i (x^k, s_i^k) - \cT_i (x^\star, s_i^k)} + \sqn{\cT_i (x^\star, s_i^k) - h^k_i} }.        \label{eq:lma-VR-distance-recursion}
    \end{eqnarray}
\end{lemma}
\begin{proof}
    Conditionally on $x^k, h^k_1, \ldots, h^k_n, s_1^k, \ldots, s_n^k$ we have,
    \begin{eqnarray}
        \ecn{x^{k+1} - x^\star} &=& \ecn{ \br{1 - \eta} x^k + \frac{\eta}{n} \sum_{i=1}^{n} \br{\delta^{k}_i + h^k_i} - x^\star } \nonumber \\
        &\overset{\eqref{eq:basic-fact-variance-decomp}}{=}& \sqn{ \br{1 - \eta} x^k + \eta \cT (x^k, s^k) - x^\star } + \frac{\eta^2}{n^2} \ecn{ \sum_{i=1}^{n} \delta^k_i - \ec{\delta^k_i} } \nonumber \\
        &\overset{\eqref{eq:variance-of-independent-sum}}{=} & \sqn{ \br{1 - \eta} x^k + \eta \cT (x^k, s^k) - x^\star } + \frac{\eta^2}{n^2} \sum_{i=1}^{n} \ecn{ \delta^k_i - \ec{\delta^k_i} } \nonumber \\
        & \leq & \sqn{ \br{1 - \eta} x^k + \eta \cT (x^k, s^k) - x^\star } + \frac{\eta^2 \omega}{n^2} \sum_{i=1}^{n} \sqn{ \cT_i (x^k, s_i^k) - h^k_i } . \nonumber
    \end{eqnarray}
    We now take expectation with respect to the randomness in $s_1^k, \ldots, s_n^k$ and conditionally on $x^k, h^k_1, \ldots, h^k_n$:
    \begin{eqnarray}
        \label{eq:lma-VR-distance-recursion-1}
        \ecn{x^{k+1} - x^\star} &\leq & \ecn{ \br{1 - \eta} x^k + \eta \cT (x^k, s^k) - x^\star } + \frac{\eta^2 \omega}{n^2} \sum_{i=1}^{n} \ecn{ \cT_i (x^k, s_i^k) - h^k_i }.
    \end{eqnarray}
    To bound the first term in \eqref{eq:lma-VR-distance-recursion-1} we use the convexity of the squared norm as follows,
    \begin{eqnarray}
        \ecn{ \br{1 - \eta} x^k + \eta \cT(x^k, s^k) - x^\star } &= & \ecn{ \br{1 - \eta} \br{x^k - x^\star} + \eta \br{\cT (x^k, s^k) - x^\star} } \nonumber \\
        &\overset{\eqref{eq:jensen-sqnorm}}{\leq} & \br{1 - \eta} \sqn{x^k - x^\star} + \eta \ecn{\cT(x^k, s^k) - x^\star} \nonumber \\
        &\overset{\eqref{eq:T-shrinkage-to-constant}}{\leq} & \br{1 - \eta + \eta \br{1 - \rho}} \sqn{x^k - x^\star} + \eta B \nonumber \\
        \label{eq:lma-VR-distance-recursion-2}
        &=& \br{1 - \eta \rho} \sqn{x^k - x^\star} + \eta B.
    \end{eqnarray}
    For the second term in \eqref{eq:lma-VR-distance-recursion-1} we have,
    \begin{eqnarray}
        \label{eq:lma-VR-distance-recursion-3}
        \ecn{ \cT_i (x^k, s_i^k) - h^k_i } &\overset{\eqref{eq:sqnorm-triangle-inequality}}{\leq} & 2 \ecn{\cT_i (x^k, s_i^k) - \cT_i (x^\star, s_i^k)} + 2 \ecn{\cT_i (x^\star, s_i^k) - h^k_i}.
    \end{eqnarray}
    It remains to substitute with \eqref{eq:lma-VR-distance-recursion-2} and \eqref{eq:lma-VR-distance-recursion-3} in \eqref{eq:lma-VR-distance-recursion-1}:
    \begin{eqnarray*}
        \ecn{x^{k+1} - x^\star} & \leq & (1 - \eta \rho) \sqn{x^k - x^\star} + \eta B \\
        && \qquad + \frac{2 \eta^2 \omega}{n^2} \sum_{i=1}^{n} \ec{ \sqn{\cT_i (x^k, s_i^k) - \cT_i (x^\star, s_i^k)} + \sqn{\cT_i (x^\star, s_i^k) - h^k_i} }.
    \end{eqnarray*}
\end{proof}

We now prove Theorem~\ref{theorem:VR-distributed-fpci}.
By Lemmas~\ref{lemma:VR-distance-recursion} and \ref{lemma:VR-hk-recursion} taking conditional expectation w.r.t. $x^k, h^k_1, \ldots, h^k_n$,
\begin{eqnarray}
    \ec{\Psi^{k+1}} &=& \ecn{x^{k+1} - x^\star} + \frac{4 \eta^2 \omega}{\alpha n^2} \sum_{i=1}^{n} \ecn{h^{k+1}_{i} - x^\star} \nonumber \\
       & \overset{\eqref{eq:lma-VR-distance-recursion} + \eqref{eq:lma-VR-hk-recursion}}{\leq}& \br{1 - \eta \rho} \sqn{x^k - x^\star} + \frac{6 \eta^2 \omega}{n^2} \sum_{i=1}^{n} \ecn{\cT_i (x^k, s_i^k) - \cT_i (x^\star, s_i^k)}  \nonumber\\
        && \qquad + \frac{4 \eta^2 \omega}{\alpha n^2} \br{1 - \frac{\alpha}{2}} \sum_{i=1}^{n} \ecn{h^k_i - \cT_i (x^\star, s_i^k)} + \eta B
 \nonumber \\
&
        \overset{\eqref{eq:individual-Ti-Lipschitz}}{\leq}& \br{1 - \eta \rho} \sqn{x^k - x^\star} + \frac{6 \eta^2 \omega}{n^2} \sum_{i=1}^{n} c_i^2 \cdot \sqn{x^k - x^\star}  \nonumber\\
        && \qquad + \frac{4 \eta^2 \omega}{\alpha n^2} \br{1 - \frac{\alpha}{2}} \sum_{i=1}^{n} \ecn{h^k_i - \cT_i (x^\star, s_i^k)} + \eta B
 \nonumber \\
       & =& \br{1 - \eta \rho + \frac{6 \eta^2 \omega c^2}{n}} \sqn{x^k - x^\star} + \eta B \nonumber  \\
        && \qquad + \frac{4 \eta^2 \omega}{\alpha n^2} \br{1 - \frac{\alpha}{2}} \sum_{i=1}^{n} \ecn{h^k_i - \cT_i (x^\star, s_i^k)}.  \label{eq:thm-VR-main-proof-1}
\end{eqnarray}

To get the optimal stepsize $\eta \in (0, 1]$ we solve
\[ \min_{\eta \in (0, 1]} \pbr{ 1 - \eta \rho + \frac{6 \eta^2 \omega c^2}{n} }, \]
One can observe that the solution of this problem is the value of $\eta$ in Theorem~\ref{theorem:VR-distributed-fpci}. Using this choice of $\eta$ we get
\begin{equation}
    \label{eq:thm-VR-main-proof-2}
    1 - \eta \rho + \frac{6 \eta^2 \omega c^2}{n} = 1 - \frac{\eta \rho}{2} - \frac{\eta \rho}{2} \br{1 - \frac{12 \eta \omega c^2}{n \rho}} \overset{}{\leq} 1 - \frac{\eta \rho}{2}.
\end{equation}
Hence using \eqref{eq:thm-VR-main-proof-2} in \eqref{eq:thm-VR-main-proof-1},
\begin{eqnarray}
    \ec{\Psi^{k+1}} &\leq& \br{1 - \frac{\eta \rho}{2}} \sqn{x^k - x^\star} + \eta B + \frac{4 \eta^2 \omega}{\alpha n^2} \br{1 - \frac{\alpha}{2}} \sum_{i=1}^{n} \ecn{h^k_i - \cT_i (x^\star, s_i^k)} \nonumber \\
    &\leq & \max \pbr{ 1 - \frac{\eta \rho}{2}, 1 - \frac{\alpha}{2} } \ec{ \sqn{x^k - x^\star} + \frac{4 \eta^2 \omega}{\alpha n^2} \sum_{i=1}^{n} \sqn{h^k_i - \cT_i (x^\star, s_i^k)}  } + \eta B \nonumber \\
    \label{eq:thm-VR-main-proof-3}
    &=& \br{1 - \frac{\min \pbr{\alpha, \eta \rho}}{2}} \Psi^k + \eta B.
\end{eqnarray}
It remains to take unconditional expectations in \eqref{eq:thm-VR-main-proof-3}, yielding \eqref{eq:thm-VR-main-convergence}.

\end{document}